\documentclass{article}

\usepackage{arxiv}

\usepackage[utf8]{inputenc} 
\usepackage[T1]{fontenc}    
\usepackage[hidelinks]{hyperref}       
\usepackage{url}            
\usepackage{booktabs}       
\usepackage{amsfonts}       
\usepackage{nicefrac}       
\usepackage{microtype}      
\usepackage{graphicx}
\usepackage{natbib}
\usepackage{doi}
\usepackage{balance} 
\usepackage{multirow}
\usepackage{amsthm}
\usepackage{amsmath}
\usepackage{enumitem}
\usepackage{xcolor}
\usepackage[english]{babel}

\usepackage{times}
\usepackage{soul}
\usepackage{caption}
\usepackage{amsthm}
\usepackage{array}
\usepackage{algorithmic}
\usepackage[switch]{lineno}

\usepackage[linesnumbered,ruled,vlined]{algorithm2e}
\usepackage{acronym}
\usepackage{bbm}
\usepackage{multicol}
\usepackage{subcaption}

\newacro{GenCo}[GenCo]{generating companies}
\newacro{DR}[DR]{demand response}
\newacro{MAB}[MAB]{multi-armed bandit}
\newacro{PowerTAC}[PowerTAC]{Power Trading Agent Competition}
\newacro{DC}[DC]{distribution companies}
\newacro{RP}[RP]{reduction probability}
\newacro{RR}[RR]{reduction rate}

\renewcommand{\paragraph}[1]{\noindent\textbf{#1}}
\newcommand{\ourmodel}{\textsc{ExpResponse}}
\newcommand{\ouralgo}{\textsc{MJS--ExpResponse}}
\newcommand{\ourucb}{\textsc{MJSUCB--ExpResponse}}

\newcommand{\wouralgo}{\textsc{WeightedMJS--ExpResponse}}

\SetCommentSty{mycommfont}

\SetKwInput{KwInput}{Input}                
\SetKwInput{KwOutput}{Output}              

\newtheorem{theorem}{Theorem}
\newtheorem{lemma}{Lemma}

\title{A Novel Demand Response Model and Method for Peak Reduction in Smart Grids -- PowerTAC}

\author{ {\hspace{1mm}Sanjay Chandlekar} \\
	International Institute of Information \\
        Technology (IIIT), Hyderabad, India \\
	\texttt{sanjay.chandlekar@research.iiit.ac.in} \\
	\And
	{\hspace{1mm}Arthik Boroju} \\
	Indian Institute of Technology,\\ 
        Ropar, India \\
	\texttt{arthikvishwakarma@gmail.com} \\
	\And
        {\hspace{1mm}Shweta Jain} \\
	Indian Institute of Technology,\\ 
        Ropar, India \\
	\texttt{shwetajain@iitrpr.ac.in} \\
        \And
        {\hspace{1mm}Sujit Gujar} \\
	International Institute of Information \\
        Technology (IIIT), Hyderabad, India \\
	\texttt{sujit.gujar@iiit.ac.in} \\
}

\date{}

\renewcommand{\headeright}{}
\renewcommand{\undertitle}{Accepted as an Extended Abstract in AAMAS'23}
\renewcommand{\shorttitle}{Sanjay Chandlekar et al.}

\hypersetup{
pdftitle={A Novel Demand Response Model and Method for Peak Reduction in Smart Grids -- PowerTAC},
pdfsubject={cs.GT, econ.TH},
pdfauthor={Sanjay ~Chandlekar, Arthik ~Boroju, Shweta ~Jain, Sujit ~Gujar},
pdfkeywords={Demand Response (DR) in Smart Grids, PowerTAC, Learning Customer DR Model, Peak Reduction in Smart grids},
}

\begin{document}
\maketitle
\begin{abstract}
One of the widely used peak reduction methods in smart grids is \emph{demand response}, where one analyzes the shift in customers' (agents') usage patterns in response to the signal from the distribution company. Often, these signals are in the form of incentives offered to agents. This work studies the effect of incentives on the probabilities of accepting such offers in a real-world smart grid simulator, PowerTAC. We first show that there exists a function that depicts the probability of an agent reducing its load as a function of the discounts offered to them. We call it reduction probability (RP). RP  function is further parametrized by the rate of reduction (RR), which can differ for each agent. We provide an optimal algorithm, \ouralgo, that outputs the discounts to each agent by maximizing the expected reduction under a budget constraint. When RRs are unknown, we propose a Multi-Armed Bandit (MAB) based online algorithm, namely \ourucb, to learn RRs. Experimentally we show that it exhibits sublinear regret. Finally, we showcase the efficacy of the proposed algorithm in mitigating demand peaks in a real-world smart grid system using the PowerTAC simulator as a test bed.
\end{abstract}


\section{Introduction} \label{sec:intro}
     \emph{Load balancing} is one of the most prevalent problems in energy grids, which occurs when there is a sudden surge of consumption (i.e., during peak hours) and the demand goes beyond the normal working range of supply. The sudden surge in demand leads to multiple issues: (i)  peak demands put an added load on electricity \emph{\ac{GenCo}} to supply additional energy through fast ramping generators to fulfill the energy requirement of the customers (agents). (ii) The grid needs to support such dynamics and peak demand. The ramping up of the generators results in higher costs for \ac{DC}. Typically, daily peak demands are approximately $1.5$ to $2.0$ times higher than the average demand~\cite{eia}. As per one estimation, a $5\%$ lowering of demand during peak hours of California electricity crisis in $2000/2001$ would have resulted in $50\%$ price reduction~\cite{IEA}. Figure~\ref{fig:dr_effect} conveys the same idea where a slight reduction in peak demand can significantly bring down the net electricity costs. Thus, it is paramount to perform load balancing in the grid efficiently.

    A promising technology for load balancing is a \emph{smart grid}. It is an electricity network that supplies energy to agents via two-way digital communication. It allows monitoring, analysis, control, and communication between participants to improve efficiency, transparency, and reliability~\cite{smart_grid_def}. The smart grid technology is equipped with smart meters capable of handling the load in the smart grid by advising the agents to minimize energy usage during heavy load scenarios. The smart grid system can effectively balance the load by incentivizing agents to shift their energy usage to non-peak timeslots by signaling them the updated tariffs, commonly known as \ac{DR}.

    \begin{figure}[!t]
          \centering \includegraphics[width=0.4\linewidth]{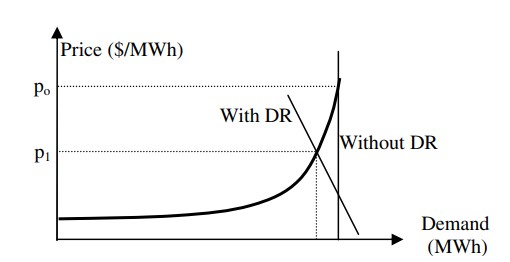}
          \caption{Effect of Demand Response (DR) on Energy Prices [Albadi and El-Saadany, 2007]} 
          \label{fig:dr_effect}
    \end{figure}
        
    \ac{DR} involves \ac{DC} offering the agents voluntarily monetary incentives to optimize their electricity load. There are many approaches, such as auction-based mechanisms~\cite{7279144,7218655} and dynamic pricing~\cite{GOUDARZI2021103073} to achieve \ac{DR}. The major challenge with these approaches is that different agents may respond differently to the given incentives. Thus, to increase agent participation, it becomes crucial to learn their reaction toward these incentives. Learning agents' behavior is challenging due to the uncertainty and randomness that creeps in due to exogenous factors like weather~\cite{shweta2020multiarmed,Li18}. 
    Works like~\cite{shweta2020multiarmed,Li18} consider a very simplistic model -- when \ac{DC} offers to an agent incentive more than what it values, the agent reduces every unit of electricity it consumes with a certain probability independent of the incentive. This probability is termed as \ac{RP}~\cite{JAIN14,shweta2020multiarmed}. \ac{RP}s are learned using \ac{MAB} solutions. There are three primary issues with these approaches. (i) Agents' valuations need to be elicited~\cite{JAIN14,shweta2020multiarmed}, which adds additional communication complexity, (ii) agents reduce all with \ac{RP} else nothing, and (iii) \ac{RP}s do not change with incentives. In the real world, an increase in incentives should lead to an increase in RP. Our work considers the model where the \ac{RP} is a function of incentives offered and not a constant for an agent, and reduction is not binary.
    
    To model RP as a function of incentive, we need to carry out experiments with smart grids. However, any \ac{DR} technique (or such experiments) proposed for a smart grid should also maintain the grid's stability. The only way to validate that the proposed technique would not disrupt the grid operations while achieving \ac{DR} is to test it on real-world smart grids, which is practically impossible. Nevertheless, \ac{PowerTAC}~\cite{KETTER20132621} provides an efficient and very close-to real-world smart grid simulator intending to facilitate smart grid research. We first perform experiments with \ac{PowerTAC} to observe the behavior of different agents for the offered incentives. With rigorous experiments, we propose our model \ourmodel. We observe that the agents respond quickly to the incentives; however, more incentives may not substantially increase reduction guarantees. Different agents may have a different rate of reduction (RR) to incentives that determine how fast RP changes w.r.t. incentives. It also models the consumer valuation for one unit of electricity. A higher RR corresponds to the case where a consumer values the electricity less (for example, a home consumer). In contrast, a lower RR value indicates that the consumer values the electricity higher (for example, an office consumer). 
        
    We propose an optimization problem for the DC to maximize the expected peak reduction within the given budget. We then provide an optimal algorithm, namely \ouralgo, for the case when the \ac{RR}s of the agents are known. When \ac{RR}s are unknown, we employ a standard \ac{MAB}-algorithm, \ourucb, to learn \ac{RR}s. Our experiments with synthetic data exhibit sub-linear regret (the difference between the expected reduction with known RRs and the actual reduction with \ourucb). With this success, we adopt it for PowerTAC set-up and experimentally show that it helps in reducing peak demands substantially and outperforms baselines such as distributing budget equally across all agent segments. In summary, the following are our contributions,
    \begin{itemize}
        \item We propose a novel model (\ourmodel) which mimics smart grid agents' demand response (DR) behavior by analyzing agents' behavior in a close-to real-world smart grid simulator, \ac{PowerTAC}.
        \item We design an offline algorithm to optimally allocate the budget to agents to maximize the expected reduction.
        \item We design an online algorithm based on a linear search method to learn the \ac{RR} values required to calculate optimal allocation in the offline algorithm. We further show that the proposed algorithm exhibits sub-linear regret experimentally.
        \item We evaluate the proposed algorithm on the \ac{PowerTAC} platform -- close to a real-world smart grid system. Experiments showcase the proposed algorithm's efficacy in reducing the demand peaks in the \ac{PowerTAC} environment (14.5\% reduction in peak demands for a sufficient budget).
    \end{itemize}
    

\section{Related Work} \label{sec:replated_work}

Many demand response methods are available in the literature. Some of the popular ones include time-of-day tariff~\cite{Ramchurn11,jain13}, direct load control ~\cite{hsu1991dispatch}, the price elasticity of demand approach (dynamic pricing) ~\cite{chao2012competitive} approaches. These approaches are quite complex for the agents as the price keeps changing. It can lead to agent confusion due to uncertain supply, volatile prices, and lack of information. Due to the complexity involved in these methods, many recent works have focused on providing incentives to the agents, which make them shift their load from peak hours to non-peak hours~\cite{park2015designing,JAIN14}. 
    
In the literature, many techniques for providing incentives primarily focus on the setting where when given an offer (incentive), the consumer can either reduce or choose not to reduce the consumption. For example, \ac{DR} mechanism in ~\cite{JAIN14,shweta2020multiarmed} considered a setting where each consumer was associated with two quantities: (i) valuation per unit of electricity, which represents how much a consumer values the unit of electricity, and (ii) acceptance rate, which denotes the probability of accepting the offer if a consumer is given the incentive more than his/her valuation. The authors then proposed a Multi-Armed Bandit mechanism that elicits the valuation of each consumer and learns the acceptance rate over a period of time. Similar approaches were also considered in ~\cite{Ma16,Ma17,methenitis2019,Li18}. All the above models, in principle, assume that the acceptance rate is independent of the incentives given to the agents. In practice, this assumption does not hold. The acceptance rate ideally should increase with the increase in incentives. To the best of our knowledge, this paper considers the dependency of increased incentives on the acceptance rate for the first time, esp. in MAB-based learning settings. In principle, the paper considers the problem of an optimal allocation of the budget to different types of agents to maximize the overall peak reduction.

Two sets of works aim to maximize the peak reduction under a budget constraint. (i) With a mixed integer linear programming (MILP) approach~\cite{chen2020online},  and (ii) with an efficient algorithm by drawing similarities from the min-knapsack problem~\cite{singh2021designing}. Other than that, there are a few tariff strategies for \ac{PowerTAC} environment which mitigates the demand peaks by publishing tariffs to incentivize customers to shift their non-priority electricity usage to non-peak timeslots~\cite{ijcai2022p23,croc18,Ghosh2019}. However, none of this technique talks about \ac{DR} in detail.
    

\section{Preliminaries and Mathematical Model} \label{sec:prelim}

    In a smart grid system, distributing companies (\ac{DC}) distributes the electricity from \ac{GenCo} to agents (household customers, office spaces,  electric vehicles, etc.) in the tariff market. The customers are equipped with autonomous agents/bots to interact with the grid. Hence, we refer to customers as agents henceforth. Depending on their type, each agent exhibits a certain usage pattern which is a function of a tariff offered by the \ac{DC} for most agents. We consider $N = \{1, 2, \ldots, n\}$ agents available to prepare for \ac{DR} at any given timeslot. 

    A DR model can further incentivize agents, offering $c_i$ to agent $i$, to shift their usages from peak to non-peak timeslot. 
    However, agents may do so stochastically, based on external random events and the offered incentives. For each agent $i$, this stochasticity can be modeled by associating the probability of reducing demand in the desired timeslot as  $i$. 
    We call this probability as \emph{reduction probability} (RP) $p_i(c_i)$. Note that the reduction in electricity consumption at peak slot for agents is \emph{not binary}. For example, an agent with the usage of $10$ KWh and RP ($p_i(c_i)$) of $0.6$ would reduce its usage by $6$ KWh in expectation. The general intuition is that higher incentives lead to a higher probability of accepting the offer, reducing the load in peak hours. Typically the \ac{DC} has a limited budget $b$ to offer discounts. It  aims to achieve the maximum possible peak reduction within the budget.
    
    First, we need to model the agent's RP function $p_i(\cdot)$. We need a simulator that can efficiently model real-world agents' usage patterns and the effects of \ac{DR} on their usage patterns. \ac{PowerTAC}~\cite{KETTER20132621} replicates the crucial elements of a smart grid, including state-of-the-art customer models.  Below, we explain experimental details and observations from the \ac{PowerTAC} experiments that helped us to come up with our novel model of the \ac{RP} function. 
        
    \subsection{Modelling the Reduction Probability (RP) Function Inspired from PowerTAC} \label{ssec:powertac_inspiration}
    
       \noindent\paragraph{\ac{PowerTAC} Set-up:} The \ac{PowerTAC} simulates the smart-grid environments in the form of games that run for a fixed duration. The standard game duration is around $60$ simulation days, which can be modified to play for an even longer duration. The simulation time is discretized into timeslots corresponding to every hour of the day.
         For each game, the \ac{PowerTAC} environment randomly selects the weather of a real-world location, and the agents mold their usage pattern based on the selected weather in the game. During the game, \ac{DC} aims to develop a subscriber base in the tariff market by offering competitive tariffs, which could be \emph{fixed price} (FPT), tiered, \emph{time-of-use} (ToU) or a combination of all. The \ac{DC} also satisfies the energy requirement of their subscriber base by buying power in the wholesale market by participating in day-ahead auctions. There are different types of customers in \ac{PowerTAC}. But, we focus on \ac{PowerTAC}'s consumption agents -- who consume electricity and aim to learn their \ac{RP} function.
        
        \noindent\paragraph{Experimental Set-up:} We perform the following nine sets of experiments to model the \ac{RP} function. We play $10$ different games for $180$ simulation days for each experimental set-up and report the statistics averaged over these $10$ games. For each experiment, we make \ac{DC} publish a tariff at the start of the game and keep that tariff active throughout the game. The initial tariff rates depend on the \ac{DC} electricity purchase cost and may vary from game to game.
        
        \noindent\emph{FPT-Set-up to identify peak slots:} We make \ac{DC} publish an FPT and record each consumption agent's true usage pattern without any external signals from \ac{DC}. Based on the true usage pattern of each agent, we identify the potential peak demand hours in a day. Figure~\ref{fig:fpt_tou} shows the usage pattern of a PowerTAC agent in response to the FPT; in this figure, the hours $7$ and $17$ have the peak usages during the day. The rate value of the FPT is derived by adding a profit margin in the \ac{DC}'s electricity purchase cost. Next, we study the agents' response to different tarrifs. To this, we consider the ToU model, where different prices are proposed at different times. These prices, however, are the same for all agents.
        
        \noindent\emph{ToU-Set-up:} In ToU tariffs, the rate charged for each unit of electricity consumed can vary depending on the time of the day. The ToU tariffs are designed so that the agents get discounts during non-peak hours and no/little discounts during peak hours. The average rate of the tariffs across all timeslots remains the same as the previous FPT-set-up. Essentially, all the ToU tariffs have the exact same area under the curve (AUC) as the FPT. We perform such an experiment for the remaining $8$ sets by offering discounts in each set; we give $x\%$ discount on non-peak timeslots compared to the price in peak timeslots. Here $x \in \{1, 2, 5, 7.5, 10, 15, 20, 30\}$. Figure~\ref{fig:fpt_tou} explains how we move from an FPT (Fig.~\ref{fig:fpt_tou}(a)) to a ToU tariff (Fig.~\ref{fig:fpt_tou}(c)) by offering a certain discount and keeping the AUC the same for all the tariffs. Based on the discount level, the agents modify their usage patterns ((Fig.~\ref{fig:fpt_tou}(b,d)), and we collect the usage data of each agent for each of the sets. 
        
        \begin{figure}[t!]
          \centering
          \includegraphics[width=0.6\linewidth]{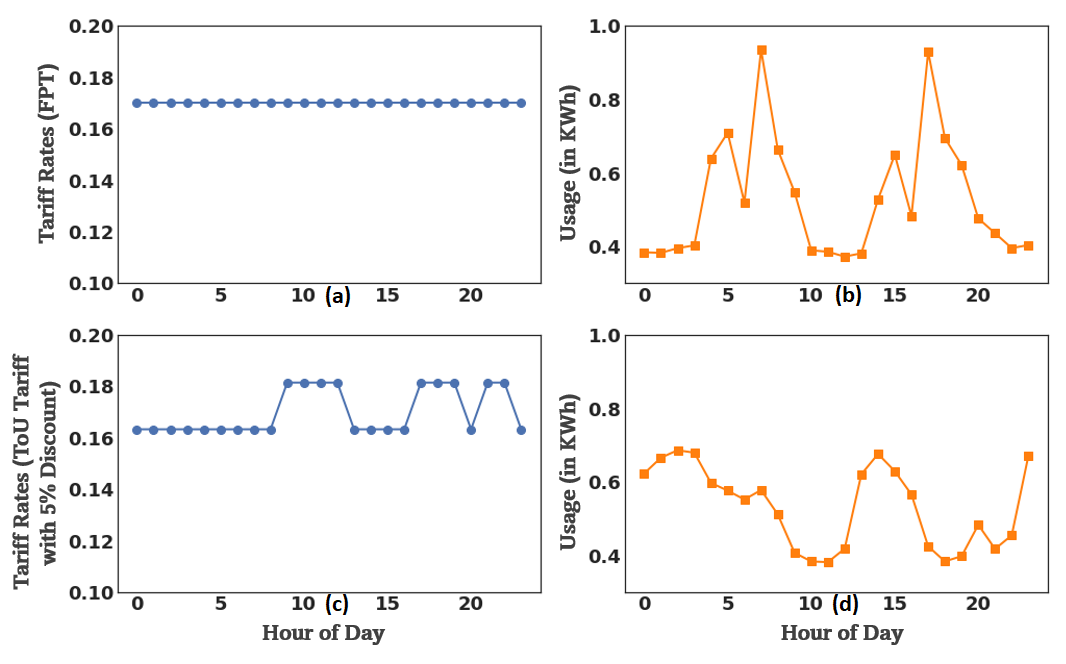}
          \caption{PowerTAC customer's response for FPT and ToU tariffs}
          \label{fig:fpt_tou}
        \end{figure}

        \begin{figure*}[t!]
        \centering
          \begin{subfigure}{0.33\textwidth}
          \includegraphics[width=1.0\textwidth]{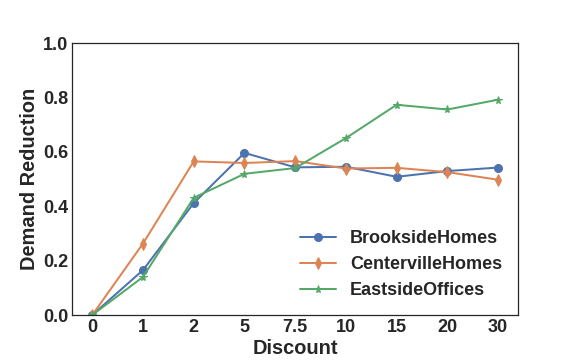}
          \caption{} \label{fig:powertac_agents_DR_peak1}
          \end{subfigure}
          \begin{subfigure}{0.33\textwidth}
          \includegraphics[width=1.0\textwidth]{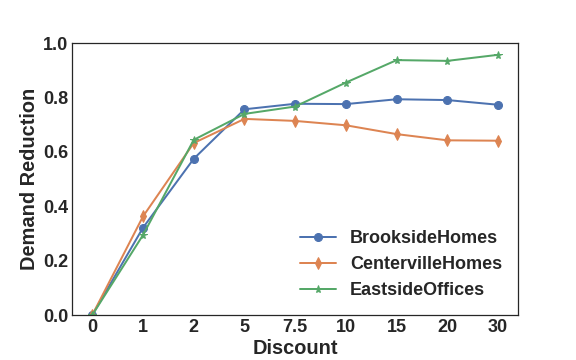}
          \caption{} \label{fig:powertac_agents_DR_peak2}
          \end{subfigure}
          \begin{subfigure}{0.33\textwidth}
          \includegraphics[width=1.0\textwidth]{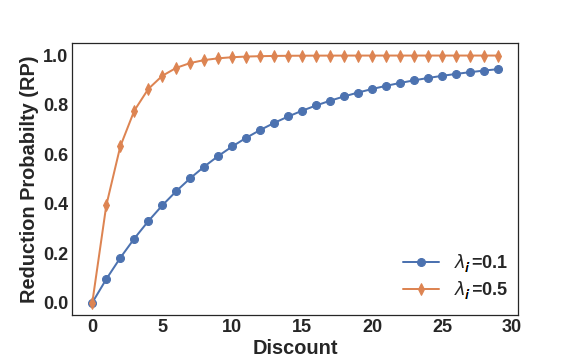}
          \caption{}\label{fig:prob_DR}
          \end{subfigure}
          \caption{DR probability function of PowerTAC customers for (a) the highest peak, (b) the $2$nd highest peak, and (c) exponential probability function for \ourmodel}
        \end{figure*}

        To analyze the effects of various discounts on agents' usage patterns, we pick the top two peak hours in the day for each agent. Then, we calculate the difference between the electricity usage during FPT-Set-up and electricity usage during discounted ToU-Set-up for both peak slots. We do this for all eight sets numbered from $2$ to $9$. We can view the discounted tariffs as a \ac{DR} signal for the agents to shift their non-priority usages from peak to non-peak timeslots. Below, we show the observations of the \ac{DR} experiments for a few selected agents.
        
        Figure~\ref{fig:powertac_agents_DR_peak1} and Figure~\ref{fig:powertac_agents_DR_peak2} show the \ac{DR} behavior of three \ac{PowerTAC} agents BrooksideHomes, CentervilleHomes and EastsideOffices for their top two peak demand hours (re-scaled to visualize peak reduction as a probability function). The first two agents are household customers, whereas the last agent is an office customer. Analysing the plots gives a crucial insight into the agents' behavior. The agents reduce their usage by a great extent for the initial values of discount $(1\%$, $2\%$ and $5\%)$ but cannot reduce their usage further even when offered a much higher discount; secondly, different agents follow the different rate of reduction. 
        
        Based on the \ac{PowerTAC} experiments, we conclude that the reduction probability function can be modeled by an exponential probability function and is given as: 
        \begin{equation}
            p_i(c_i) = 1 - e^{-\lambda_i c_i}, \forall{i} \in N \label{eq:red_prob}
        \end{equation}
        Here, $c_i$ is a discount (or incentive) given to agent $i$, and $\lambda_i$ is its reduction rate (\ac{RR}). The proposed function depends upon the choice of $\lambda_i$; the higher value of $\lambda_i$ generates a steeply increasing curve (as shown with $\lambda_i=0.5$), while the lower $\lambda_i$ value makes the curve increase slowly with each discount (as shown with $\lambda_i=0.1$) as shown in Figure~\ref{fig:prob_DR}. Let $c=(c_1,c_2,\ldots,c_n)$ and $\lambda=(\lambda_1,\lambda_2,\ldots,\lambda_n)$ be vector of offered incentives and \ac{RR}s. 
        
    \subsection{\ourmodel: The Optimization Problem} \label{ssec:optimization_problem}
        We assume that all the agents have the same electricity consumption in peak slots\footnote{Agents consuming different amounts can be trivially modeled by duplicating agents}.
        The aim is to maximize the expected reduction under a budget constraint. This leads to the following optimization problem:
        \begin{align}
        \label{eq:opt-final}
         \textbf{max}_{c_i} \: \Large \sum_{i=1}^{n} (1 - e^{-\lambda_i c_i})\; s.t. \sum_{i=1}^n c_i \le b
        \end{align}    
        
        Suppose the \ac{RR} ($\lambda$) values are known. In that case, we present an optimal algorithm \ouralgo\ to efficiently distribute the budget $b$ among the agents to maximize the expected sum of peak reduction (Section~\ref{ssec:offline}). When \ac{RR}s are unknown, we provide \ourucb\ algorithm that estimates it (Section~\ref{ssec:online}). The algorithm is motivated by multi-armed bandit literature~\cite{JAIN14,shweta2020multiarmed} and uses the linear search over the possible range of values of \ac{RR}. 
    
    \section{Proposed Algorithms for \ourmodel} \label{sec:dr_algo}
    
        This section proposes a novel algorithm to solve \ourmodel. We discuss two settings: (i) perfect information that assumes the knowledge of \ac{RR}, and (ii) imperfect information  where \ac{RR} values need to be learned over time.
        
    \subsection{Perfect Information Setting: Known RR} \label{ssec:offline}
    
        
        \begin{algorithm}[t!]
        \DontPrintSemicolon          
          \KwInput{Budget $b$, $n$, \ac{RR} vector $\lambda$}
          \KwOutput{Final Allocation vector $c$}

          $cost \leftarrow 0, c \leftarrow 0_n$; \tcp*{initialization}
          \While{$cost \le b$}
          {
                $d \leftarrow 1$ \tcp*{iterator}
                $l \leftarrow 1$ \tcp*{index of agent with largest jump}
                \While{$d \le n$}
                {
                    $\Delta_{d}^{c_d+1} \leftarrow (1 - e^{-\lambda_d(c_d+1)}) - (1 - e^{-\lambda_dc_d})$ \\
                    $\Delta_{l}^{c_l+1} \leftarrow (1 - e^{-\lambda_l(c_l+1)}) - (1 - e^{-\lambda_lc_l})$ \\

                    \If{$\Delta_{d}^{c_d+1} > \Delta_{l}^{c_l+1}$}
                        {$l = d$}
                    $d = d + 1$
                }
                
            $c_l = c_l + 1$ and $cost = cost + 1$ 
          }
          \textbf{return} \textrm{$c$} \tcp*{final allocation}
        
        \caption{\ouralgo\ Algorithm} \label{alg:offline}
        \end{algorithm}
        
        \ouralgo\  (Algorithm~\ref{alg:offline}) distributes one unit of budget to an appropriate agent in each iteration until the entire budget is exhausted. To decide the appropriate agent for the current iteration, we calculate \emph{jump} ($\Delta$) values for all the agents. We define $\Delta_{i}^{j+1}$ value for each agent as the change in \ac{RP} for a unit change in discount. For example, if an agent $i$ has \ac{RP} of $(1 - e^{-\lambda_i(j+1)})$ for discount $c_i = (j+1)$ and \ac{RP} of $(1 - e^{-\lambda_ij})$ for discount $c_i = j$, then the jump $\Delta_{i}^{j+1}$ is the difference between these two probabilities. The algorithm finds an agent $l$ having the maximum such jump for the current unit of reduction (denoted by $\Delta_{l}^{c_l+1}$ for agent $l$) and allocates the current unit discount to agent $l$--\emph{Maximum Jump Selection} (MJS). 
        Finally, the algorithm returns the allocation, which is the optimal distribution of the initial fixed budget, as shown in the below theorem. 
        \begin{theorem}
        \ouralgo\ is optimal.
        \end{theorem}
        \begin{proof}
        For any discount vector $c$, the objective function in Equation \ref{eq:opt-final} can be written as a sum of jumps $\Delta_i^j$ which denote the additional increase in reduction probability of consumer $i$ when offered $j$ units of discount compared to $j-1$. i.e.
        \allowdisplaybreaks
        \begin{align*}
            &\max_{c}\sum_{i=1}^n 1-e^{-\lambda_ic_i}\\ &= \max_{c}\sum_{i=1}^n\left(\sum_{j=1}^{c_i} (1-e^{-\lambda_ij}) - (1-e^{-\lambda_i(j-1)})\right)\\ &= \sum_{j=1}^{c_i} \Delta_i^j \max_{c}\sum_{i=1}^n \sum_{j=1}^{c_i}\Delta_i^j\:\:\:\: s.t.\:\: \sum_{i=1}^n c_i \le b
        \end{align*}
        Thus at the optimal solution, one unit is allocated to $b$ highest jumps and $0$ to other jumps. We now need to prove that the earlier jump is higher than the latter, i.e., $\Delta_i^l \ge \Delta_i^j\ \forall l < j$. The below lemma proves this for any agent $i$.
        \end{proof}
        
        \begin{lemma}
        For each $i$, we have $\Delta_i^j \ge \Delta_i^{j+1}$ with $\lambda_i \ge 0$
        \end{lemma}
        \begin{proof}
        We have the following:
        \begin{align*}
            &e^{-j\lambda_i}(e^{\lambda_i} -1 ) \ge e^{-j\lambda_i}(1 - e^{-\lambda_i})\\
            \Rightarrow &e^{-\lambda_i(j-1)} - e^{-\lambda_ij} \ge e^{-j\lambda_i} - e^{-\lambda_i(j+1)}
        \end{align*}
        From the last equation, we have $\Delta_i^j \ge \Delta_i^{j+1}$.
        \end{proof}

        
    \noindent  Note that one can use KKT conditions and derive a set of linear equations to determine an optimal distribution of $b$. Our proposed algorithm is simple, determines an optimal solution in linear time, and has a time complexity of $O(nb)$. 
    
    \subsection{Imperfect Information Setting: Unknown RR} \label{ssec:online}
        As \ac{RR} of the agents are unknown in this setting, we estimate them based on the history of the agents, which consists of the agents' response during the past timeslots. For each agent, we store its historical behavior by keeping track of the offered history and success history; we estimate $\hat{\lambda}\mbox{ and }\hat{\lambda}^+$ through a routine $LinearSearch(\cdot)$. 

        \begin{algorithm}[t!]
        \DontPrintSemicolon
          
          \KwInput{Budget $b$, $n$, Batch Size $bS$,  $T$}
          \KwOutput{Allocation $\{c_t\}_{t=1}^T$}

          Initialize $ \hat\lambda ,\hat\lambda^+$ randomly\tcp*{$n$-dimensional vectors}
          
          Initialize $offeredInst$, $successInst$, $offeredHist$ and $successHist$ to 0          \tcp*{2D matrices of size $n \times b$}
          
          $t \leftarrow 0$

          \While{$t < T$}
          {
            $\{c_{t'}\}_{t'=t}^{t+bs} = \ouralgo(b, n, \hat\lambda^+)$ \\
            
            \For{$i = 1\rightarrow n$}
            {
                \If{$c_t(i) \ne 0$}
                {
                    $offeredInst(i, c_t(i))$ += $bS$
                    
                    $successInst(i, c_t(i))$ += \# Successes for agent $i$
                }
            }
            Update $Hist=\{Hist, offeredInst,successInst\}$ \\
            Clear $offeredInst,successInst$\\
            $t \leftarrow t+bS$ \\
            $[\hat\lambda, \hat\lambda^+] \leftarrow LinearSearch(Hist, n, b, t)$ \\
            
          }
          \textbf{return} \textrm{$\{c_t\}_{t=1}^T$}
        
        \caption{\ourucb\ Algorithm} \label{alg:online}
        \end{algorithm}
        
        \noindent\paragraph{\ourucb} We start by initializing $\hat\lambda$ and its UCB component $\hat\lambda^+$ and then estimating $\hat{p}_i(c_i)$ for each agent $i$ and for each $c_i$ using $Linearsearch(\cdot)$ at each timeslot. 
       \begin{figure*}[t!]
          \centering
          \includegraphics[width=\textwidth]{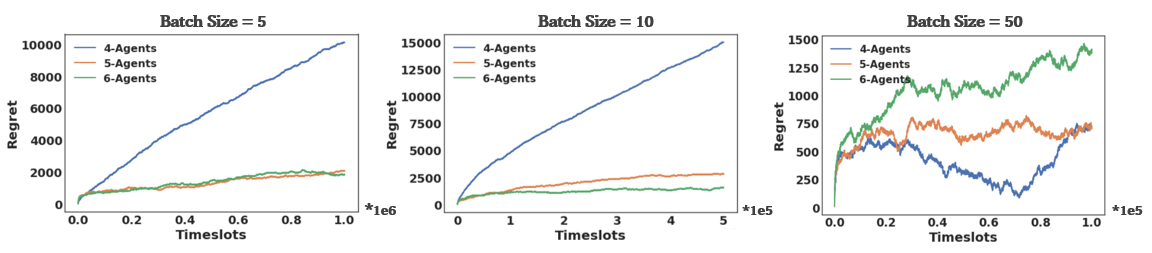}
          \caption{Comparing \ourucb's regret over $25$ iterations with varying batch sizes [budget=5, T=5M]}
          \label{fig:batch}
        \end{figure*}

        \begin{figure}
          \centering
          \includegraphics[width=0.7\linewidth]{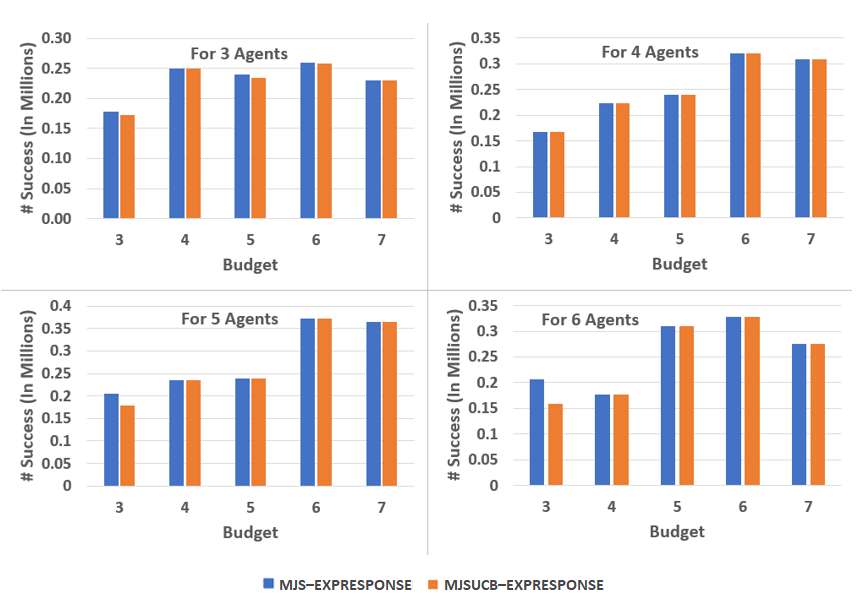}
          \caption{Comparing \ourucb\ and \ouralgo\ with varying budget and number of agents [Iterations = 25, T=5M]}
          \label{fig:agents_budget}
         \end{figure}
    
        \noindent\paragraph{LinearSearch():} Estimating \ac{RR}s with the offered history and success history,  we calculate $\hat{p}_i(c_i) = \frac{SuccHist(i,c_i)}{OfferredHist(i,c_i)}$, for each $c_i$ offered to the agent $i$. $\hat{p}_i(c_i)$ is then used to calculate candidate values of \ac{RR} using Equation~\ref{eq:red_prob}. Based on the candidate \ac{RR} values, we determine $\hat{\lambda}_i$ that minimizes the squared error loss between the historical probabilities and the probabilities calculated based on the Equation~\ref{eq:red_prob} i.e., $\hat{\lambda}_i \in argmin_l \sum_{c_i\in [b]} \left(\hat{p}_i(c_i) - (1 - e^{-l c_i})\right)^2$, for each of the discount value. The \ac{RR} value that achieves the least squared error loss is returned as the optimal \ac{RR} after the current timeslot. We follow the same method for each of the agents. Algorithm~\ref{alg:online} discusses our proposed \ourucb\ method in more detail, which takes budget $b$, $n$, batch Size $bS$, and $T$ as inputs and returns $\hat\lambda$s. Here,  $\hat\lambda,\hat\lambda^+$ denote estimated $\lambda$ and its UCB version, respectively. 
        
    \subsection{Experimental Evaluation of \ourucb}
    To check whether the algorithm converges to the true \ac{RR}, we conduct extensive analysis on a simplified version of a smart grid. Here, we discuss the experimental set-up to observe the regret of the proposed \ourucb. Regret is the difference between total reduction with known RRs and total reduction with unknown RRs. In both experiments, we repeat the experiment $25$ times, each instance having independently chosen $\lambda$. We report different statistics averaged over 25 iterations. 
    
    \noindent\paragraph{Exp1-- Effect of batch size:} In this experiment, we keep the budget $b$ and $T$ constant, and vary batch sizes. This experiment shows the change in regret behavior as we change the batch sizes from low to high. For each of the batch sizes, we compare the regret for a different number of agents. 
    Figure~\ref{fig:batch} compares average regret of \ourucb\ over $25$ iterations of varying true \ac{RR} values of agents, with varying batch size, and keeping budget $b = 5$ and $T = 5M$ constant. The figure shows three subplots with batch sizes of $5$, $10$, and $50$, respectively. Each subplot compares the regret values for $4$, $5$, and $6$ agents, respectively, and shows \emph{sub-linear} regret in the case of $5$ and $6$ agents for three different batch sizes. With an increased batch size to $50$, even the case with $4$ agents converges to sub-linear regret within a few timeslots. 
    
    \noindent\paragraph{Exp2-- Effect of budget and relation w.r.t. $T$:} The second set of experiments is similar to the Exp1, except in this set, we vary the budget and number of agents while keeping the number of iterations and $T$ constant. The second set of experiments compares the peak reduction achieved by \ourucb\ and the optimal peak reduction when we know all the agents' true \ac{RR}. MAB literature refers to it as \emph{regret}. It also shows how the success rates change when we increase the initial budget keeping the number of agents the same. Additionally, the experiment helps us observe the peak reduction with varying budgets across different numbers of agents. Figure~\ref{fig:agents_budget} compares \ourucb\ peak reduction achieved when we know the true \ac{RR} and peak reduction achieved by \ourucb\ across varying budget and the varying number of agents over $25$ iterations of varying true \ac{RR} values, here we keep $iterations\: T = 5M$. As shown in the figure, the total peak reduction achieved by \ourucb\ is nearly the same as the reduction we get when we know all agents' true \ac{RR} and allocate the budget optimally. Thus, we analytically conclude that \ourucb\ achieves a sub-linear regret and its peak reduction success rates are approximately the same as the optimal peak reduction success rates. We next show the performance of \ourucb\ in \ac{PowerTAC}.
    
    \section{\ourucb\ in PowerTAC} \label{sec:ucb_powertac}

    \noindent\paragraph{Modelling the customer groups}: In the algorithm, we assume all the agents are of the same type (meaning they use the same amount of electricity). However, in \ac{PowerTAC}, agents are of varied types; for example, some belong to the household agents class, some belong to the office agents class, and some are village agents. Even for office agents, some offices use a high amount of electricity compared to others. Thus, we begin by grouping the agents based on their electricity usage, and create $4$ such groups, namely, $G1$, $G2$, $G3$ and $G4$, as shown in Table~\ref{tab:customer_groups}. We consider groups due to the limitations of \ac{PowerTAC}, where we cannot publish individual customer-specific tariffs. We can only publish tariffs for customer groups having similar usage ranges. However, our proposed model and algorithms do not rely on any assumption of the existence of such groups and treat each consumer as a separate user (in Sections \ref{sec:prelim} and \ref{sec:dr_algo}). We leave out the remaining \ac{PowerTAC} agents as they do not use a considerable amount of electricity in the tariff market.
    
    \noindent\paragraph{Designing the tariffs for each group}: For each group $Gi$, we publish ToU tariff $ToUi$ such that agents from $Gi$ subscribe to tariff $ToUi$, and no other group of agents subscribe to that tariff. To achieve this, we combine ToU tariffs with \emph{tier} tariffs as follows. In \ac{PowerTAC}, tier tariffs specify rate values and upper bounds on electricity usage below which the specified rates are applicable. However, if the usage goes beyond that particular bound, the agent has to pay the rate values associated with the next higher bound. As we have segregated the agents based on their usage range, for any targeted group, we offer standard ToU rate values for its particular usage range and high rates for the remaining ranges of electricity usage. Thus, a group of agents naturally like the tariff designed for their group as the other tariffs are way costlier for their usage pattern. At any moment in the PowerTAC game, we keep all four $ToU$ tariffs active (one for each group); these tariffs keep getting updated based on the \ac{DR} signals from \ac{DC}.
    \begin{table} 
        \small
        \centering
        \caption{Customer groups detail}
        \begin{tabular}{ | c | c | c| c | } 
        \hline
        Group & Customers & Type & \%Usage in Tariff Market  \\
        \hline
        \hline
        $G1$ &BrooksideHomes \& CentervilleHomes & Household & $50\%$ \\ 
        \hline
        $G2$ & DowntownOffices \& EastsideOffices & Small Offices & $25\%$ \\
        \hline
        $G3$ & HextraChemical & Mid-level Offices & $10\%$ to $12\%$  \\
        \hline
        $G4$ & MedicalCenter-1 & High-level Offices & $10\%$ to $12\%$  \\
        \hline
        \end{tabular}
        \label{tab:customer_groups}
    \end{table}
    
    \noindent\paragraph{Adapting \ourucb\ in \ac{PowerTAC}}: While proposing our model, we assume that agents are identical and have the same usage capability. Thus, maximizing the sum of probability would also result in maximizing reduction. However, for general smart grid settings such as \ac{PowerTAC}, we modify our model by giving weightage to agents based on their usage percentage (market cap). Higher weightage is given to agents that can reduce the larger amount of energy. We modify \ouralgo\ to introduce weights proportional to groups' contribution to electricity usage for each group. $weights = \{4,\:2,\:1,\:1\}$ to groups $\{G1,\:G2,\:G3,\:G4\}$, respectively in our experiments. We still use $c=\ouralgo(\cdot)$ (Line 5, Algorithm~\ref{alg:online}) to find the group that can fetch the highest increase in the probabilities as shown in Algorithm~\ref{alg:offline}.
    
    While allocating discounts to the groups, instead of allocating a $1$ unit of budget to each group, we weigh the unit with the group's weight. For example, if $G1$ gets selected for the discount, we assign a $4$ unit discount instead of $1$. We call this way of allocation as \wouralgo. It may help to assign weights to the groups as assigning weights will allocate discounts proportional to their peak reduction capacity. For instance, $10\%$ reduction in $G1$ would reduce more peak demand than $10\%$ reduction in $G4$. 
    
    \noindent\paragraph{Creating baseline:} To compare the performance, we consider the baseline of uniformly allocating the budget to all the groups. This leads to publishing group-specific tariffs with equal discounts. We record the peak reduction efficiency and reduction in capacity transactions from the baseline strategy. We then use the recorded information as a benchmark to evaluate \ourucb\ performance. Furthermore, we compare the efficacy against the strategy when we do not provide groups with any \ac{DR} signals.
    
    \noindent\paragraph{Evaluation metrics:} Finally, we evaluate \ourucb's performance on two metrics, (i) \ourucb's peak demand reduction capability, which indicates how much percentage of peak demand reduction \ourucb\ achieved compared to the benchmark strategies, and (ii) the reduction in \emph{capacity transaction} penalties that suggest how effectively \ourucb\ can mitigate such penalties compare to the benchmark strategies.

    \noindent\textbf{Capacity transactions} In \ac{PowerTAC}, capacity transactions are the penalties incurred by the \ac{DC} if the agents subscribed to their portfolio contribute to the peak demand scenarios. These huge penalties are a way to penalize the \ac{DC} for letting the agents create demand peaks. Thus, as opposed to the previous section where we analytically show \ourucb\ exhibits a sub-linear regret, here in PowerTAC experiments, we aim to reduce capacity transaction penalties of \ac{DC} using \ourucb.
  
\subsection{Experiments and Discussion}
    \noindent\paragraph{Experimental set-up:} We perform multiple experiments with different initial budgets. We play $8$ games in each set with approximately $28$ simulation weeks (total $210$ weeks). For each set, we start the experiments by randomly initializing \ac{RR} values for each group and calculate the budget allocation based on \wouralgo\ as well as \ouralgo\ (line 5 in \ourucb), called as \ourucb-W and \ourucb-UW, respectively. 
    
    As explained in Section~\ref{sec:ucb_powertac}, for each of the $4$ groups, we have four ToU tariffs. We keep the same tariffs active for $3$ simulation days and invoke the \ourucb\ at the end of the $3$rd day. Based on the success probabilities, we update the $offeredHist$ and $successHist$, and calculate the next set of $\hat\lambda$ and $\hat\lambda^+$ values for each group. Using the new $\hat\lambda^+$, we calculate the next demand allocation and publish the new tariffs as explained earlier. While publishing new tariffs, we revoke the previous ones; thus, only $4$ tariffs are active at any time in the game. This $3$ days process constitutes a single learning iteration ($t$). To calculate the success probability of each tariff, we played $10$ offline games without any discount to any group and noted down the top two peak timeslots. Let $x_{i,1} $ and $x_{i,2}$ denote per group usage during those peak timeslots. Then, we compute the success probability as $p_{i,1} = (1 - y_{i,1} / x_{i,1})$ and $p_{i,2} = (1 - y_{i,2} / x_{i,2})$, with $y_{i,1}$ and $y_{i,2}$ denoting group 1 and 2 usage respectively. $p_i$ is then set as $p_i = \frac{p_{i,1}+p_{i,2}}{2}$. We perform $2$ sets of experiments with  $b=15\%$ and $b=7.5\%$. We define a scalar value that gets multiplied by the discounts to generate fractional discounts. 
    
    \noindent\paragraph{Observations and Discussion:} Table~\ref{tab:compare_reduction_cap_trans} shows the cumulative peak usages under \ourucb\ and bench-marked (baseline) method for the top $2$ peaks of groups $G1$ to $G4$. As shown in the table, for the overall $210$ simulation weeks of training in \ac{PowerTAC}, \ourucb\ cumulative peak usage for peak1 is similar to the baseline method for both weighted and unweighted allocations, while slightly worse than the baseline for the peak2. The observation is consistent for the budget values $15\%$ and $7.5\%$. However, if we focus on only the last $10$ weeks of training, \ourucb's peak usage reduction capabilities are visible. Both weighted and unweighted allocations achieve cumulative peak reduction close to $14.5\%$ concerning \emph{No Discount} peak usages for peak1 and $b = 15\%$, which is almost $5$ times better than the baseline while maintaining similar performance as the baseline for peak2. Similarly, \ourucb\ achieves significant improvement for $b = 7.5\%$ too for peak1 by reducing the peaks $3$ to $4$ times better than baseline. Furthermore, as shown in Table~\ref{tab:compare_reduction_cap_trans}, capacity transaction penalties in the last $10$ weeks are significantly lower than the \emph{No Discount} and baseline. Due to DR signals, agents sometimes shift some of the demand from peak1 to peak2 or cannot reduce any demand from peak2. However, if the overall system's performance is observed with the help of capacity transaction penalties in PowerTAC experiments, the penalties are significantly lower than the baseline, reinforcing the efficacy of \ourucb\ in the \ac{PowerTAC} environment.   
    
    \begin{table} [t!] 
        \small
        \centering
        \caption{Peak usage comparison (usage in MWh) and Capacity transaction comparison (avg. penalty)}
        \begin{tabular}{ |c|c|c|c|c|c|c|c|}
        \hline
        \multirow{2}{*}{Method} & \multicolumn{2}{|c|}{$b = 15$} & \multicolumn{2}{|c|}{$b = 7.5$} & \multirow{2}{*}{$b = 15$} & \multirow{2}{*}{$b = 7.5$}\\
        \cline{2-5}
        & P1 & P2 & P1 & P2 & & \\
        \hline
        \hline
        No Discount & $70.2$ & $69.7$ & $70.2$ & $69.7$ & $249355$ & $249355$\\ 
        \hline
        Baseline & $68.1$ & $67.4$ & $67.8$ & $67.7$ & $233768$ & $227070$\\
        \hline
        \hline
        \multicolumn{7}{|c|}{Average Over All $210$ Weeks of Training} \\
        \hline
        \ourucb-W & $68.4$ & $67.8$ & $68.8$ & $69.5$ & $233643$ & $228478$ \\
        \hline
        \ourucb-UW & $68.3$ & $68.1$ & $68.1$ & $68.6$ & $233248$ & $229467$ \\
        \hline
        \hline
        \multicolumn{7}{|c|}{Average Over Last $10$ Weeks of Training} \\
        \hline
        \ourucb-W & $60.2$ & $69.8$ & $64.1$ & $70.9$ & $226374$ & $228351$ \\ 
        \hline
        \ourucb-UW & $60.0$ & $67.6$ & $61.4$ & $69.0$ & $225775$ & $228276$ \\ 
        \hline
        \end{tabular}
        \label{tab:compare_reduction_cap_trans}
    \end{table}
        
\section{Conclusion}
    The paper proposed a novel DR model where the user's behavior depends on how much incentives are given to the users. Using the experiments on the \ac{PowerTAC} real-world smart grid simulator, we first showed that agents' probability of accepting the offer increases exponentially with the incentives given. Further, each group of agents follows a different rate of reduction (RR). Under the known RR setting, we proposed \ouralgo\, which leads to an optimal allocation of a given budget to the agents, which maximizes the peak reduction. When RRs are unknown, we proposed \ourucb\ that achieves sublinear regret on the simulated data. We demonstrated that \ourucb\ is able to achieve a significant reduction in peak demands and capacity transactions just within 200 weeks of simulation on \ac{PowerTAC} simulator.

\bibliographystyle{unsrt}

\end{document}